\documentclass[11pt,a4paper]{article}
\usepackage{enumitem}
\usepackage[square,numbers]{natbib}
\usepackage[linesnumbered,ruled,vlined]{algorithm2e}
\usepackage{amsmath}
\usepackage{amsfonts}
\usepackage{amssymb}
\usepackage{authblk}
\usepackage{amsthm}

\newcommand{\TDUCB}{TD-UCB}
\newcommand{\SWGREEDY}{SW-GREEDY}
\newcommand{\ext}{\textsc{\tiny EXT}}
\newcommand{\pst}{\textsc{\tiny PST}}
\newcommand{\ri}{\textsc{\tiny RI}}

\usepackage{tikz}
\usetikzlibrary{calc}
\usepackage{xcolor}

  \renewenvironment{thebibliography}[1]{%
    \begin{oldthebibliography}{#1}%
      \setlength{\parskip}{0ex}%
      \setlength{\itemsep}{0ex}%
  }%
  {%
    \end{oldthebibliography}%
  }
 \relax
\usepackage{tikz}
\usetikzlibrary{calc,fit}

\makeatletter
\tikzset{%
  remember picture with id/.style={%
    remember picture,
    overlay,
    save picture id=#1,
  },
  save picture id/.code={%
    \edef\pgf@temp{#1}%
    \immediate\write\pgfutil@auxout{%
      \noexpand\savepointas{\pgf@temp}{\pgfpictureid}}%
  },
  if picture id/.code args={#1#2#3}{%
    \@ifundefined{save@pt@#1}{%
      \pgfkeysalso{#3}%
    }{
      \pgfkeysalso{#2}%
    }
  }
}

\def\savepointas#1#2{%
  \expandafter\gdef\csname save@pt@#1\endcsname{#2}%
}

\def\tmk@labeldef#1,#2\@nil{%
  \def\tmk@label{#1}%
  \def\tmk@def{#2}%
}

\tikzdeclarecoordinatesystem{pic}{%
  \pgfutil@in@,{#1}%
  \ifpgfutil@in@%
    \tmk@labeldef#1\@nil
  \else
    \tmk@labeldef#1,(0pt,0pt)\@nil
  \fi
  \@ifundefined{save@pt@\tmk@label}{%
    \tikz@scan@one@point\pgfutil@firstofone\tmk@def
  }{%
  \pgfsys@getposition{\csname save@pt@\tmk@label\endcsname}\save@orig@pic%
  \pgfsys@getposition{\pgfpictureid}\save@this@pic%
  \pgf@process{\pgfpointorigin\save@this@pic}%
  \pgf@xa=\pgf@x
  \pgf@ya=\pgf@y
  \pgf@process{\pgfpointorigin\save@orig@pic}%
  \advance\pgf@x by -\pgf@xa
  \advance\pgf@y by -\pgf@ya
  }%
}
\newcommand\tikzmark[2][]{%
\tikz[remember picture with id=#2] #1;}
\makeatother
\newcommand\MyBoxZ[4][-1ex]{%
  \tikz[remember picture,overlay,pin distance=0cm]
  {\draw[draw=#4,line width=1pt,fill=#4!20,rectangle,rounded corners]
( $ (pic cs:#2) + (-1ex,2.3ex) $ ) rectangle ( $ (pic cs:#3) + (23ex,#1) + (1.7ex,-1.3ex)$ );
}
}

\newcommand\MyBoxY[4][-1ex]{%
  \tikz[remember picture,overlay,pin distance=0cm]
  {\draw[draw=#4,line width=1pt,fill=#4!20,rectangle,rounded corners]
( $ (pic cs:#2) + (-1ex,2.3ex) $ ) rectangle ( $ (pic cs:#3) + (35ex,#1) $ );
}
}

\newcommand\MyBoxX[4][-1ex]{%
  \tikz[remember picture,overlay,pin distance=0cm]
  {\draw[draw=#4,line width=1pt,fill=#4!20,rectangle,rounded corners]
( $ (pic cs:#2) + (-0.3ex,2.3ex) $ ) rectangle ( $ (pic cs:#3) + (6.1ex,#1) $ );
}
}
\newtheorem{remark}{Remark}
\newtheorem{lemma}{Lemma}
\newtheorem{theorem}{Theorem}
\newtheorem{definition}{Definition}
\newcommand{\barbelow}[1]{\mkern3mu\underline{\mkern-1mu #1\mkern-1mu}\mkern1mu}
\newenvironment{prf}{\noindent{\bf Proof:}\hspace*{1em}}

\makeatletter
\newcommand{\removelatexerror}{\let\@latex@error\@gobble}
\makeatother
\usepackage[pdfborder={0 0 0}]{hyperref}% For email addresses
\usepackage{cleveref}
\begin{document}
\title{A Truthful Mechanism with Biparameter Learning for Online Crowdsourcing}

\author[1]{Satyanath Bhat}
\author[1]{Divya Padmanabhan}
\author[1]{Shweta Jain}
\author[1]{Y Narahari}

\affil[1]{Computer Science and Automation \\
Indian Institute of Science \\
\url{http://lcm.csa.iisc.ernet.in/}}
%\date{\today}

\maketitle

\begin{abstract}
We study a problem of allocating divisible jobs, arriving online, to workers in a crowdsourcing setting which involves learning two parameters of strategically behaving workers. Each job is split into a certain number of tasks that are then allocated to workers. Each arriving job has to be completed within a deadline and each task has to be completed satisfying an upper bound on probability of failure. The job population is homogeneous while the workers are heterogeneous in terms of costs, completion times, and times to failure. The job completion time and time to failure of each worker are stochastic with fixed but unknown means. The requester is faced with the challenge of learning two separate parameters of each (strategically behaving) worker simultaneously, namely, the mean job completion time and the mean time to failure. The time to failure of a worker depends on the duration of the task handled by the worker. Assuming non-strategic workers to start with, we solve this biparameter learning problem by applying the Robust UCB algorithm. Then, we non-trivially extend this algorithm to the setting where the workers are strategic about their costs. Our proposed mechanism is dominant strategy incentive compatible and ex-post individually rational with asymptotically optimal regret performance.
\end{abstract}

% Note that the category section should be completed after reference to the ACM Computing Classification Scheme available at
% http://www.acm.org/about/class/1998/.

%A category including the fourth, optional field follows...
%\category{D.2.8}{Software Engineering}{Metrics}[complexity measures, performance measures]

%General terms should be selected from the following 16 terms: Algorithms, Management, Measurement, Documentation, Performance, Design, Economics, Reliability, Experimentation, Security, Human Factors, Standardization, Languages, Theory, Legal Aspects, Verification.
%
%
%\category{I.2.11}{Distributed Artificial Intelligence}{Intelligent agents}
%
%\terms{Algorithms, Economics, Theory}
%
%\keywords{Strategic Agents, Mechanism Design}
%
%%Keywords are your own choice of terms you would like the paper to be indexed by.

\section{Introduction}
Crowdsourcing is widely used in procuring labels and services for traditional AI applications. Often many of the tasks crowdsourced are more readily accomplished by humans than computers. An additional advantage is the scalable and cost-effective nature of crowdsourcing.
However, typical crowdsourcing platforms may not consider several important aspects of traditional planning such as ensuring work completion within a strict deadline and with assured guarantees on the quality. 

As a motivating example for this paper, consider a sequence of jobs arriving online where each job corresponds to translating  a large document which has to be completed within a deadline and with an assured level of accuracy. It may not be possible for a single individual worker to accomplish this job, so the requester could split such a job into tasks (either at chapter or section or any other level) and allocate each task to a crowd worker. Due to the very nature of the task, a worker, if employed for a long duration, might start committing errors. We refer to the duration until which an agent works without committing any error as the time to failure (TTF). Also, each worker differs in the time taken to complete the entire job (if the entire job is executed by the worker). The time taken by a worker to complete the job all by himself is called the job completion time (JCT) of that worker. Each worker incurs a certain cost to complete the entire job. Note that the workers are heterogeneous in terms of their costs, JCT, and TTF. Moreover, JCT and TTF of the workers are stochastic. An additional non-trivial challenge occurs when crowd workers are strategic and may misrepresent their costs in the hope of gaining higher utility.  This setting occurs in other problems such as tagging of a large repository of images, audio transcriptions, etc.

%%%%%%%%%%%%%%%%%%%%%%%%%%%%%%%%%%%%%%%%%%%%%%%%%%%%%%%%%%%
%%%%%% The following para could be commented out - not required
%To meet the aforementioned design goals, the requester could choose to divide the job into  tasks and assign these tasks to a larger worker set. The size of each such minuscule task can be chosen small enough to ensure that the allocated agent is less prone to making an error. As these tasks are distributed to different agents and completed simultaneously, the deadline requirement would also be met. However, this naive approach does not consider the heterogeneity of costs due to which the cost of such a solution can be prohibitively high.  Moreover, the stochastic parameters associated with fatigue and completion time are unknown to the requester and need to be learnt. The learning problem inherent in the cost optimization gets harder in the presence of strategic crowd agents as their privately held costs have to be elicited. 
%%%%%%%%%%%%%%%%%%%%%%%%%%%%%%%%%%%%%%%%%%%%%%%%%%%%%%%%%%

In this work, we consider jobs which (a) arrive online, (b) are divisible (into tasks),  (c) have strict completion deadlines, and (d) are to be completed with an assured accuracy.  We propose a multi-armed bandit (MAB) mechanism which learns the two parameters (mean job completion time (MJCT) and mean time to failure (MTTF)) of the workers while eliciting their privately held costs truthfully. We show that the proposed MAB mechanism minimizes the regret while meeting the deadline and accuracy requirements on every job. The following are the specific contributions of this work. 

\begin{enumerate}[leftmargin=*]
\item \emph{\textbf{Non-strategic, with learning:}}
We look at the problem of allocating divisible online jobs to crowd workers so as to meet the constraints on deadline and accuracy (\Cref{sec:non-strategic}). The underlying optimization problem turns out to be non-trivial since the parameters MJCT and MTTF of the workers are unknown. We overcome this challenge by devising a biparameter learning scheme based on the Robust UCB~ algorithm \cite{DBLP:BUBECK12HEAVYTAIL}. Further, we embed this learning scheme into our social welfare maximizing algorithm, which we refer to as \SWGREEDY.

\item \emph{\textbf{Strategic, with learning:}} We next non-trivially extend the results above to the setting where worker costs are privately held (\Cref{sec:strategic}) by designing a mechanism (\TDUCB). This mechanism is is dominant strategy incentive compatible  and ex-post individually rational (\Cref{thm:ic-ir}).
\item \emph{\textbf{Regret Analysis:}} In \Cref{sec:regret-analysis}, we show, for non-strategic as well as strategic settings, that the number of jobs for which a non-optimal worker set is chosen, is upper bounded by $O(\log T)$ (\Cref{thm:optpull}), where $T$ is the total number of jobs to be completed.  Moreover, once an optimal worker set is selected, the allocation algorithm converges asymptotically to an efficient allocation, ensuring that the average regret goes to zero in the limit (\Cref{thm:asym}).  
\item \emph{\textbf{Simulations:}} Finally, we show the practical efficacy of our learning mechanism via simulations in \Cref{sec:simulations}. 
\end{enumerate}

%%%% YN --- This Secion needs to be strengthened 

\section{Previous Work}
We now look at previous work related to our setting. We group the relevant literature based on whether or not crowd workers are strategic.

In the non-strategic case, most of the work in crowdsourcing has focused on  models for aggregating labels and building classifiers \cite{RAYKAR10,KARGER11}. Many efforts also address problems similar to the one considered in our paper. \citet{FARADANIHI11} look at the design of pricing schemes dependent on the completion times of the workers.
% They model the arrival times of the workers as a non-homogeneous Poisson process and learn these stochastic parameters. The pricing to the workers are a function of these learned parameters. 
The strategic nature of the workers is not considered here. The problem of completing tasks within a deadline is also investigated by \citet{YU15}. The authors consider the setting where the workers delegate tasks to other workers when they are unable to complete the work within a deadline. 
%The problem boils down to deciding when to sub-delegate tasks, identify the worker to whom the task must be sub-delegated, and also the quantity of work to be allocated.
Here the costs to workers are assumed to be known and workers are non-strategic.
Under a different setup, \citet{ding2013multi} look at the budgeted multi-armed bandit problem where the two parameters stochastic costs and stochastic rewards are learnt. However, they do not consider strategic workers. 

In the strategic case, \citet{PRAPHUL15} look at allocating indivisible tasks to strategic crowd workers under deadline constraints with the assumption that the reliability (in terms of completion of the task) of the agents is common knowledge and not estimated. \citet{SINGER13} and \citet{Biswas2015} look at pricing mechanisms in the presence of budget constraints and task completion deadlines. However, the heterogeneity with respect to time to failure is not modelled.  \citet{Tran-Thanh2013} look at crowdsourcing classification tasks with the goal of trading off cost and accuracy of the estimation. However the TTF and JCT of the workers is not modeled here. Choosing an optimal worker set in order to obtain an assured accuracy level has been studied in \citet{JainGBZN14}. The allocation algorithm makes use of the multi-armed bandits abstraction \citet{AUER02}. A version of their allocation algorithm was designed for the case where workers are strategic with respect to bidding their costs. However, their setting does not look at the completion of tasks within a deadline.
The problem of allocating tasks concurrently to several workers in order to meet deadlines is looked at by  \citet{GERDING2010}. The work uses a variant of VCG mechanism  to elicit the costs truthfully from the workers. They consider stochastic completion times of tasks but do not consider the time to failure during the allocation.

Our work differs from all the work listed above in that, we design an allocation scheme to complete jobs within a deadline while simultaneously learning the mean completion time as well as the mean time to failure of the  workers. We also design a mechanism to elicit the costs of the workers truthfully.

%%%%%% YN -- This para is to be described very clearly. Please avoid footnote. Referees will be confused with this.

\section{The Model}
Let $N=\{1,\ldots,n\}$ denote the set of crowd workers (also referred to as agents) available to the requester. A sequence of $T$ homogeneous jobs arrives at the platform, one at a time. Following are some of the design issues pertaining to the requester. 
\begin{enumerate}[leftmargin=3mm]
\item \textbf{Job Parameters} 
\begin{enumerate}[leftmargin=1mm]
\item \emph{Deadline:} The clock starts ticking for a job as soon as it arrives. We use $D$ to denote the deadline. The deadline $D$ on each job is an upper bound on the duration, starting from the arrival of that job,  before which the job is required to be completed in expectation. 
\item \emph{Task creation:} The requester can divide a current job $t$ $(t=1, \ldots, T)$ into a certain number of tasks so as to facilitate completion of the job by the deadline $D$. We use $x_i^{(t)}$ to denote the fraction of the job $t$ assigned as a task to the worker $i$. Therefore,  $0 \leq x_i^{(t)} \leq 1$ and $\sum_{i=1}^n x_i^{(t)} = 1$. We assume arbitrary division of a given job into tasks for ease of exposition. However, this assumption can be relaxed to capture meaningful constraints such as the size of the task. 
\item \emph{Threshold on probability of failure for tasks:} A worker is more likely to commit an error if he works for a longer duration on a task. We say a worker has failed when he commits an error. We use $\varepsilon$ to denote (the common) threshold on probability of failure for any task. This threshold allows the requester to control the overall ``quality" of the job.
\end{enumerate}
\item \textbf{Worker Parameters}  
\begin{enumerate}[leftmargin=1mm]
\item \emph{Job Completion Time (JCT):}  A worker has a stochastic job completion time, which is the time he requires to complete the entire job by himself. JCT for a worker is random variable with a fixed but unknown mean. We refer to the mean job completion time as MJCT. The requester wishes to learn the MJCT for each worker. If $\rho_i$ is the MJCT of worker $i$, then the task allocation $x_i^{(t)}$ will meet the deadline constraint in expectation if $x_i^{(t)} \times \rho_i \leq D$. 
\item \emph{Time to Failure (TTF):} A worker is also characterized by a stochastic time to failure, which denotes the duration for which a worker would work without a failure.  Like JCT, TTF also has a fixed yet unknown mean, which the requester wishes to learn.  If $F_i$ is the CDF of TTF for agent $i$, who workers for a expected duration $x_i^{(t)} \times \rho_i$ on the task allocation given by the fraction $x_i^{(t)}$ of job $t$, the requirement on threshold probability error dictates $F_i(x_i^{(t)} \times \rho_i) \leq \varepsilon$.
\item \emph{Cost Incurred:} Worker $i$ has a privately held cost $c_i \in [\barbelow{c},\bar{c}]$ which represents the cost incurred by worker $i$  to complete the job entirely on his own. Therefore, the cost involved to complete $x_i^{(t)}$ fraction of the job by the worker $i$ is $c_i x_i^{(t)}$. 
\end{enumerate}
\item \textbf{Goal of Optimization Problem:} The constraints on deadline and threshold on probability of failure for every task has to be met in a cost optimal way for every online job $t$. Thus, the underlying optimization problem for the entire collection of  jobs $\{1,2,\ldots, T\}$ is given by \cref{opt_problem}.
\end{enumerate}
  
\begin{equation}
\label{opt_problem}
\begin{array}{|c|}
\hline\\
\displaystyle \min_{x_i^{(t)} \in [0,1]} \quad \sum_{t=1}^T \displaystyle \sum_{i=1}^n c_ix_i^{(t)},\\ 
\text{subject to}, \\  
\sum_{i=1}^n x_i^{(t)} = 1, \forall t\\
\text{Completion time}(x_i^{(t)}) \leq D \;\forall i \in N, \forall t,\\
\text{Probability of failure}(x_i^{(t)}) \leq \varepsilon \; \forall i \in N, \forall t.
\;\\
\hline
\end{array}
\end{equation}
As mentioned earlier, the JCT and the TTF of the workers are stochastic in nature. We assume the JCT of each worker follows a log-normal distribution with unknown yet fixed mean $\rho_i \in [\barbelow{\rho}, \bar{\rho}]$ while the TTF for each worker follows an exponential distribution with mean $\beta_i \in [\barbelow{\beta},\bar{\beta}]$.
\begin{remark}[Choice of Distributions]
The choice of log-normal distribution is due to its wide applicability in social sciences and economics to model similar quantities. However any suitable non-negative random variable whose distribution is sub-Gaussian (or sub-exponential) may be used. 
As discussed, the errors in this setting are introduced due to higher working duration on the task. This is analogous to the modelling of failure as function of time, in biological or computer or reliability literature, as exponential distributions. Hence, we model the TTF of the workers as exponential. 
%We discuss other possible alternatives in the future work section.
\end{remark}
The optimization problem stated in \cref{opt_problem} involves a learning scheme along with cost minimization across all the $T$ online jobs. However, due to independence across the jobs, the problem can be decomposed into a sequential cost minimization problem corresponding to each job ($t$). Therefore, in \cref{opt_problem} the summation over the jobs can be omitted. This enables us to use $x_i$ in place of $x_i^{(t)}$ for the sequential optimization problem for each job.

\section{The Case of Non-Strategic Workers}
\label{sec:non-strategic}
We first study the scenario where the costs $c_i$ incurred by the workers are common knowledge. If the means ($\rho_i$ and $\beta_i$) are known to the requester, no feasible allocation $x_i$ to the worker $i$ should exceed $D/\rho_i$. The additional requirement on accuracy requires that the probability of a worker failing in the duration $\rho_ix_i$ does not exceed $\varepsilon$. This is equivalent to the constraint $F_i(\rho_ix_i) \leq \varepsilon$ where $F_i$ is the CDF of the random variable TTF of worker $i$ which we model as the exponential distribution with mean $\beta_i$. On simplification, the requester's optimization problem reduces to \cref{opt_problem_known}.

\begin{equation}
\label{opt_problem_known}
\begin{array}{|c|}
\hline\\
\displaystyle \min_{x_i \in [0,1]}  \displaystyle \sum_{i=1}^n c_ix_i,\\ 
\text{subject to}, \\  
\sum_{i=1}^n x_i = 1,\\
x_i \leq \frac{1}{\rho_i}\min\left(D,\beta_i \ln\left(\frac{1}{1-\varepsilon}\right)\right) \;\forall i \in N, \\
\;\\
\hline
\end{array}
\end{equation}
In practice, $\rho_i$ and $\beta_i$ are not known and need to be learnt.  We make use of the multi-armed bandit (MAB) abstraction for learning these parameters. More specifically, since $\rho_i$ and $\beta_i$ are sub-exponential distributions, we appeal to the Robust UCB technique~\cite{DBLP:BUBECK12HEAVYTAIL}. While $\psi$-UCB algorithm~\cite{DBLP:BUBECK2012REGRET} is a regret minimizing scheme for  learning the mean of sub-Gaussian distributions, for heavy tailed distributions (e.g. log normal and exponential), Robust UCB has been shown to be regret minimizing~\cite{DBLP:BUBECK12HEAVYTAIL}. We adopt the Robust UCB scheme with truncated empirical mean as the estimator.

\subsection{Difficulty in Learning $\beta_i$}
If a worker $i$, allocated a fraction $x_i$ of the job, takes time $\tau$ for completion, then $\tau/x_i$ is a sample from the distribution log-normal($\rho_i$). Therefore, every allocation contributes one such sample for the Robust UCB algorithm estimating $\rho_i$.  However, for estimating $\beta_i$, each sample allocation fed to the Robust UCB algorithm must correspond to a failure, but this is not practical as we do not observe failure at every instance of allocation. To handle this difficulty, we propose to use a surrogate random variable. Consider the experiment where a worker $i$ is allocated a task (fraction of a job) on which the worker spends a duration of at least $\delta$. The experiment is deemed to have failed if the worker $i$ fails in the first $\delta$ duration of allocation, otherwise it is deemed a success. Let $N_\delta^{(i)}$ be the number of such independent experiments till a failure is encountered. We propose to use the random variable  $\beta_{\delta,i}^{'} = \delta \times N_\delta^{(i)}$ to construct a sample from exponential($\beta_i$). To obtain such a sample, for every job $t$, we observe for a duration $\delta$ to see if any of the allocated workers have failed. Let $\eta_\delta^{(i)}$ be the number of contiguous instances (of jobs) of allocation during which a worker $i$ does not fail in the interval $\delta$. Note that $\eta_\delta^{(i)}$ is a sample from $N_{\delta}^{(i)}$. Therefore, the value $\delta \times \eta_{\delta}^{(i)}$ forms a sample of interest. Once a sample is obtained, $\eta_{\delta}^{(i)}$ is reset and the  process is again repeated to collect more samples. The expectation of the surrogate random variable in the limit coincides with $\beta_i$ due to \Cref{lemma_main}.
\begin{lemma}
\label{lemma_main}
$\lim_{\delta \rightarrow 0 }\mathbb{E}[\beta_{\delta,i}^{'}] = \beta_i$
\end{lemma}
\begin{proof}
By definition, $
\beta_{\delta,i}^{'} = N^{(i)}_\delta \times  \delta $.
%\begin{align*}
%\mathbb{E}[\beta_{\delta,i}^{'}] &= \sum_{j=1}^{N_\delta} \delta 
% \mathbb{P}\{ t_i^j> \delta\} = \mathbb{E}[N_\delta]  \mathbb{P}\{ t_i^j> \delta\}  \delta \\
% &= \mathbb{E}[N_\delta] \exp(-\delta/\beta_i) \delta
%\end{align*}
Note, $N_{\delta} \sim \text{Geometric}(1- \exp(-\delta/\beta_i))$ and therefore, $\mathbb{E}[N_\delta] = \frac{1}{1- \exp(-\delta/\beta_i)}$.
\begin{align}
\label{L-hospital-simplify}
\lim_{\delta \rightarrow 0} \mathbb{E}[\beta_{\delta,i}^{'}] &= \lim_{\delta \rightarrow 0}\frac{ \delta}{
1- \exp(-\delta/\beta_i)} = \beta_i
\end{align}
where \cref{L-hospital-simplify} follows by applying the L'Hospital's rule.
\end{proof}

\subsection{\SWGREEDY: A Greedy Allocation}
%The allocation algorithm remains same for both the strategic case (where costs are private information) and the non-strategic case (where costs are known). 
The workers are indexed in an increasing order of their costs and each worker $i$ is allocated the largest possible fraction $x^{\pst}_i$ which does not violate the constraints in \cref{opt_problem_known} till all tasks of the job are allocated.

The constraint $\frac{1}{\rho_i}\min\left(D,\beta_i \ln\left(\frac{1}{1-\varepsilon}\right)\right)$, involves means which are unknown. As mentioned earlier, we use Robust UCB to learn estimates for $\rho_i$ and $\beta_i$. $\hat{\rho}_i^+$ and $ \hat{\beta}_i^+$ are the upper confidence indices while $\hat{\rho}_i^-$ and $\hat{\beta}_i^- $ are the lower confidence indices of MJCT and MTTF respectively, obtained from Robust UCB. $\hat{\rho}_i$ and $\hat{\beta}_i$ are the empirical estimates of MJCT and MTTF respectively for worker $i$.  We could substitute $\hat{\rho}_i^{-}, \hat{\rho}_i $ or $ \hat{\rho}_i^{+}$ as the estimate for $\rho_i$ in our constraint. A higher value of $\rho_i$ enforces a lower allocation to worker $i$ compared to when a lower value of $\rho_i$ is used. Hence we refer to $\hat{\rho}_i^{+}$ as a pessimistic estimate for 
$\rho_i$. By a similar reasoning, we refer to $\hat{\beta}_i^- $ as the pessimistic estimate for $\beta_i$. The use of the pessimistic estimates ensures that even with the true underlying means the constraint in \cref{opt_problem_known} is satisfied.

The allocation algorithm discussed above ensures that the social welfare regret of the learning scheme is optimized, hence we refer to the above allocation as \SWGREEDY\ (Algorithm~\ref{al:sw-greedy}). The social welfare is defined as follows.

\begin{definition}{Social Welfare:} Social welfare of a feasible (i.e. satisfying \cref{opt_problem_known}) allocation $x$ is the sum of valuations of the agents under that allocation. In this setting, the valuation of a crowd agent is $-c_ix_i$. Therefore, social welfare is given by $\sum_{i=1}^n -c_i x_i$.
\end{definition}

Every worker $i$ is paid an amount equal to the cost incurred, i.e. $ x_i^{\pst}(t) \times c_i$, where $x_i^{\pst}$ is the allocation to agent $i$ given by Algorithm \ref{al:sw-greedy}.

\begin{remark}[Pessimistic Selection] The fundamental underlying philosophy of the UCB family of algorithms is  ``optimism under uncertainty." Intuitively, this optimism helps in adequate exploration realtive to a naive scheme which just uses the empirical estimate. In our work, we do not use this philosophy implicitly, however, due to the greedy nature of the allocation scheme, the pessimistic allocation set is a superset of the optimistic allocation. 
\end{remark}
%%%%%%%%%%%%%%%%%
%hack new page
\newpage
\MyBoxX{starta}{enda}{gray}
\MyBoxY{startb}{endb}{gray}
\MyBoxZ{startc}{endc}{gray}
\begin{algorithm}[h!]
\label{al:sw-greedy}
\DontPrintSemicolon
\caption{\SWGREEDY\ Allocation Algorithm}
\KwIn {Set of workers $N$, number of jobs $T$, deadline $D$, accuracy level $\varepsilon$, input cost vector: $c_1 \leq c_2 \leq \ldots \leq c_n$ (By re-indexing $N$)\vspace{0.04in}}
\tikzmark{starta} $\forall i \in N$, $\hat{\rho}_i = \bar{\rho}$, $\hat{\rho}_i^+ = \bar{\rho}$, $\hat{\rho}_i^- = \barbelow{\rho}$, $N_{i,t = 0}$ \;
\hspace{0.5in}$\hat{\beta}_i = \barbelow{\beta}$, $\hat{\beta}_i^+ = \bar{\beta}$, $\hat{\beta}_i^- = \barbelow{\beta}$, $N_{i,t}^\beta = 0$, \;
$\eta_\delta^{(i)}=0$ \hspace{2in} \label{initialize} \vspace{2ex} \tikzmark{enda}\;

\For{ Online job arrival $t=1,\ldots,T $}
{
\vspace{1ex}
\tikzmark{startb}$x^{\pst}(t)=\{x_1^{\pst}(t),\ldots, x_n^{\pst}(t)\} = \{0,\ldots,0\}$\;
$i=1$\;
\While{$\sum_{j=1}^n x_j^{\pst}(t) < 1$}
{
$x_i^{\pst}(t)$ = $\frac{1}{\hat{\rho}_i^+} \min \left( D, \beta_i^- \ln \left[\frac{1}{1-\varepsilon} \right]\right)$\;

\If{$1-\sum_{j=1}^{i-1} x_j^{\pst}(t)< x_i^{\pst}(t)$}
{
$ x_i^{\pst}(t) = 1-\sum_{j=1}^{i-1} x_j^{\pst}(t)$ \;
}
$i=i+1$ \tikzmark{endb} \;
}
Define $\bar{k}_t = \max\{i:x_i^{\pst} > 0\}$ \;
Allocate the job $t$ as per $x_i^{\pst}$ \;
Observe $\tilde{\tau}_i$, the time of completion of $x_i^{\pst}$ by $i$\;
$\forall i \in \{1,\ldots,\bar{k}_t \}, N_{i,t} = N_{i,t-1} +1 $\;
$\hat{\rho}_i = \left( N_{i,t-1} \times \hat{\rho}_i + \frac{\tilde{\tau}_i}{x_i^{\pst}(t)} \right) \times \frac{1}{N_{i,t}} $\;
\tikzmark{startc}\For{$i \in \{1,\ldots,\bar{k}_t\}$}
{
\If{Worker $i$ made an error during $\delta$}
{
$\hat{\beta}_i = \frac{\hat{\beta}_i \times N_{i,t-1}^\beta +  (\delta \times \eta_\delta^{(i)})}{N_{i,t-1}^{\beta} + 1}$\;
$N_{i,t}^{\beta}= N_{i,t-1}^{\beta} + 1$\;
$\eta_\delta^{(i)} = 0$\;
}
\Else
{
$\eta_\delta^{(i)} = \eta_\delta^{(i)} + 1$\;
$N_{i,t}^{\beta} = N_{i,t-1}^{\beta} $\tikzmark{endc}\;
}
}
%Obtain $\hat{\rho}_i^+,\hat{\rho}_i^-, \hat{\beta}_i^+, \hat{\beta}_i^-$ from Robust UCB \;
%\If{costs are known or $b_i = c_i$}
%{
%$p_i(t) = x_i^{\pst}(t) \times c_i$\;
%}
%\If{costs are strategic}
%{
%Pay allocated agent as per $p_i(t)$ given in \cref{eq:pymt}\;
%}
}
\end{algorithm}

\definecolor{ashgrey}{rgb}{0.7, 0.75, 0.71}
\begin{tikzpicture}[remember picture,overlay]
\coordinate (aa) at ($(pic cs:starta)+(5.7,-.8)$); % <= adjust this parameter to move the position of the annotation 
\node[rectangle,draw,white, fill=ashgrey, text width=1.5cm,align=left,right] at (aa) {Initialize};
\end{tikzpicture}

\begin{tikzpicture}[remember picture,overlay]
\coordinate (aa) at ($(pic cs:startb)+(5.7,-.55)$); % <= adjust this parameter to move the position of the annotation 
\node[rectangle,draw,white, fill=ashgrey, text width=1.5cm,align=left,right] at (aa) {Pessimistic \\ Selection};
\end{tikzpicture}

\begin{tikzpicture}[remember picture,overlay]
\coordinate (aa) at ($(pic cs:startc)+(5.1,-1.68)$); % <= adjust this parameter to move the position of the annotation 
\node[rectangle,draw,white, fill=ashgrey, text width=1.3cm,align=left,right] at (aa) {Updates for \\ Surrogate of $\beta_i$ };
\end{tikzpicture}

\section{The Case of Strategic Workers: \TDUCB}
\label{sec:strategic}
Here, before an allocation is performed, the agents announce their bids. These bids may or may not be equal to their true private costs. We denote the bid profile by $(b_i, b_{-i})$, where $b_i$ is the bid of agent $i$ and $b_{-i}$ denotes the collection of bids of all agents except agent $i$.
In order to ensure that the agents bid their costs truthfully, we introduce a mechanism \TDUCB. 
The allocation rule remains the same as the one for the case where the workers are non-strategic. We use the allocation given in Algorithm~\ref{al:sw-greedy} replacing the input costs with the bids.

\subsection{Payment Scheme}
 Let $\xi_t$ denote a tuple of allocation and performance of the allocated workers for the job $t$. The learning until job $t$ is captured in the history $h_t = \{\xi_k\}_{k=0}^{t}$. In order to specify the payment scheme, we require the notion of `externality' imposed by an agent on another. We denote the externality imposed by agent $i$ on $j$ as $x_{i,j}^{\ext}(b_i,b_{-i};h_t,t)$, which signifies the additional fraction of the job allocated to the agent $j$ in the absence of agent $i$.  The externality for the job $t$ depends on the bid profile $(b_i, b_{-i})$ as well as the history of allocations till job $t$. Let $\overline{k_t}$ be the agent with the largest reported bid in the worker set chosen by the allocation scheme. \Cref{fig:schematic-pos} provides a schematic diagram indicating the position of the bids and the agents chosen by our algorithm.
 
\begin{figure}[h!]
\centering
\includegraphics[scale=0.64]{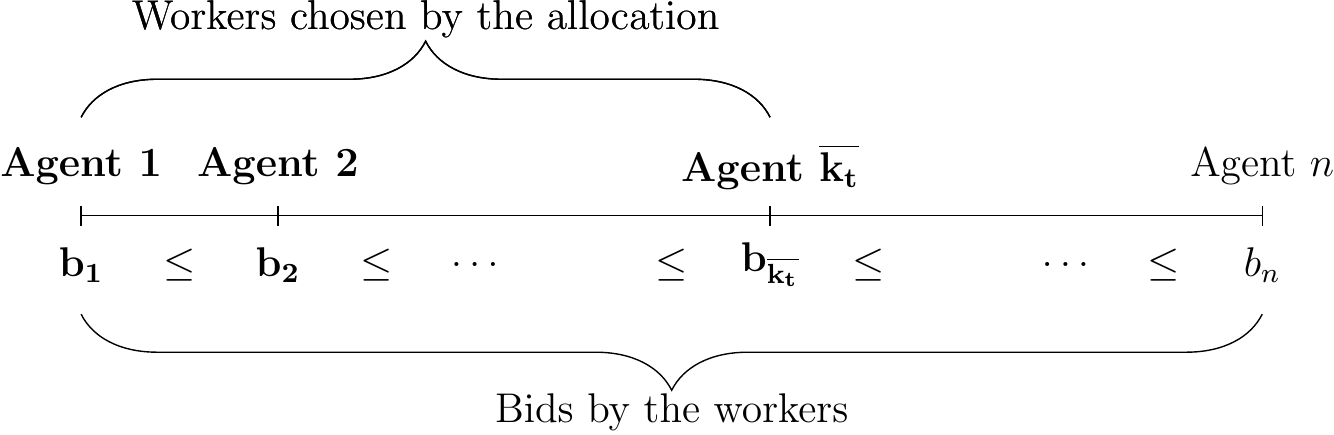}
\caption{Schematic of the allocation}
\label{fig:schematic-pos}
\end{figure} 
Formally, the externality is defined as \begin{align}
x_{i,j}^{\ext}(b_i, b_{-i};t) =
  \begin{cases}
   0     & \text{if } j< \overline{k_t} \text{ or } i> \overline{k_t},\\   
    Z_1 &\text{if } j=\overline{k_t},\\
    Z_2 &\text{if } j>\overline{k_t}, \quad\text{where},\\
  \end{cases}
\label{eq:ext} 
\end{align} 
%\[\text{where,} \qquad \qquad\qquad\qquad\qquad\qquad \qquad\qquad\qquad\qquad\]
 \[  Z_1 = \min\left(\frac{1}{\hat{\rho}_j^+} \min\left(D, \hat{\beta}_j^{-}\log\left(\frac{1}{1-\epsilon}\right)\right)- x_j^{\pst}(t), x_i^{\pst}(t)\right) , \] 
\[  Z_2 = \min\left(\frac{1}{\hat{\rho}_j^+} \min\left(D, \hat{\beta}_j^{-}\log\left(\frac{1}{1-\epsilon}\right)\right), x_i^{\pst}(t) - \displaystyle\sum_{s =\overline{k}_t}^{j-1} x_{i,s}^{\ext}(t)\right) \]

%The definition of the externality in \cref{eq:ext} leads to the payment scheme given in \cref{eq:pymt}.

We now propose a payment structure in \cref{eq:pymt} that ensures truthful bidding and positive utility to the participating agents (\Cref{thm:ic-ir}).

\begin{align}
p_i(b_i, b_{-i};t) =
  \begin{cases}
   0     & \text{if } i> \overline{k_t},\\  
    Z_3 &\text{otherwise, \quad \text{where},}
  \end{cases}
\label{eq:pymt}
\end{align}
%\[\text{where} \qquad \qquad\qquad\qquad\qquad\qquad \qquad\qquad\qquad\qquad\]
\[  Z_3 = \displaystyle\sum_{s=\overline{k}_t}^n 
\left[x_{i,s}^{\ext}(t)\times b_s\right] +  \left(x_i^{\pst}(t) - \displaystyle\sum_{s=\overline{k}_t}^n x_{i,s}^{\ext}(t) \right) \times \bar{c}   \]

\begin{remark}[Notation]
All the mechanism side parameters such as $x_i^{\pst}$ or $p_i$ are a function of $(b_i,b_{-i};h_t,t)$. Similarly, the agent side parameters such as utility $u_i$ depend on the tuple $(b_i, b_{-i};c_i,h_t,t)$. Note that the agent side parameters have an additional dependency on the true cost $c_i$. Whenever clear from the context, we drop one or more of these dependencies for ease of  notation. 
\end{remark}

%\begin{remark}[Strategic vs Non-strategic case]
%We have combined the strategic and the non-strategic versions of the problem in Algorithm \ref{tducb} as the difference arises only in the payment. For the strategic case, the payments are given by \cref{eq:pymt}. For the non-strategic case, an agent is paid an amount exactly equal to the cost incurred by him.  
%\end{remark}
\begin{remark}[Externality] Our mechanism is an externality based scheme like the VCG mechanism. We now set about the task of proving that the mechanism is truthul, regret minimizing, and individually rational,  while learning the associated stochastic parameters. Earlier works have shown the non-triviality involved in the design of such learning mechanisms~\cite{BABAIOFF09,DEVANUR09}.
\end{remark}

\subsection{Properties of \TDUCB\ Mechanism}
\begin{definition}{Utility of an Agent:} The utility of an agent in this setting is the difference between the valuation of an allocation and the payment made. The utility is given by the following.
\[ u_i(b_i,b_{-i};c_i,h_t,t) = -c_i \times x_i^{\pst}(b_i,b_{-i};c_i,h_t,t) + p_i(b_i,b_{-i};c_i,h_t,t)\]
\end{definition}
\begin{definition}{Dominant Strategy Incentive Compatible (DSIC): }A mechanism is DSIC if the utility $u_i(c_i, b_{-i}; c_i) \geq u_i(b_i, b_{-i}; c_i)\; \forall b_i \in [\underline{c}, \overline{c}],
\; \forall b_{-i} \in [\underline{c}, \overline{c}]^{n-1}, \; \forall i \in N$, where $b_i$  and $c_i$ are the bid and true cost incurred by the worker $i$ respectively, $b_{-i}$ is the bid profile of all agents other than $i$. 
\end{definition}
A DSIC mechanism ensures that an agent obtains the highest utility by bidding his true cost, irrespective of the bids of other agents.

\begin{definition}{Ex-post Individually Rational (IR): }
A mechanism is ex-post individually rational if $u_i(c_i, b_{-i}; c_i) \geq 0)$, $ \forall b_{-i} \in [\underline{c}, \overline{c}]^{n-1}\; \forall i \in N$. 
\end{definition}
An IR mechanism ensures that for every agent, the utility obtained from truthful bidding of the costs is non-negative.

\begin{theorem}
The \TDUCB\ mechanism is DSIC and IR.
\label{thm:ic-ir}
\end{theorem}
\begin{proof}
IR is immediate and follows from the definition of the payment scheme of the mechanism ({\cref{eq:pymt}}). \\
We prove the DSIC property by examining different possible scenarios of allocation for an agent.  In each of these scenarios, we compute the utilities with truthful bids as against strategic misreports of bids.  

For performing any job $t$, utility of a worker $i$ is defined as follows.
\begin{equation}
u_i(b_i, b_{-i}; c_i) = p_i(b_i, b_{-i}) - c_i \times x_i^{\pst}(b_i, b_{-i})
\end{equation} 
where $x^{\pst}_i(b_i, b_{-i})$ and $p_i(b_i,b_{-i})$ are the allocation and the payment to the worker $i$ respectively. 
We consider the following three possible scenarios for the positioning of each worker $i$ in the increasing order of ranking of the bids of the workers. We refer to the set of workers with non-zero task allocation as the active set in this proof. 
Throughout the proof, we denote by $A$ the active set of allocated workers when agent $i$ bids his true cost $c_i$. We denote by $A'$ the active set when the agent bids untruthfully.
\begin{itemize}[leftmargin=0mm, labelwidth=-2mm]
\item Case 1: $i > \overline{k_t}$\\
In this scenario, when the agent bids truthfully,
\begin{align*}
\underbrace{b_1 <b_2 < \cdots <   b_{ \overline{k_t}}}_{\text{Bids from} A} < \cdots < b_{i-1} < c_i <b_{i+1} < \cdots < c_N 
\end{align*} 
When the worker reports his cost truthfully (i.e, $b_i = c_i $), he does not receive any allocation and therefore $u_i(c_i, b_{-i}; c_i) = 0$. Now we consider the following two cases when he misreports his cost.
\begin{enumerate}[leftmargin=1mm,label=\alph*),labelwidth=-3mm]
\item \label{overbid-main}
Overbid of cost ($b_i > c_i$) : \\
Since $x_i^{\pst}( c_i, b_{-i}; c_i) = 0$, a higher bid $b_i$ would only place the agent at a position $o_i(b_i, b_{-i}) \geq o_i(c_i, b_{-i}) = i$ in the revised ranking order. At the position $o_i(b_i, b_{-i})$, again the allocation to him would be zero, that is,  $x_i^{\pst}( b_i, b_{-i}; c_i) = 0$ and thereby the utility from overbidding would be same as the utility from truthful bidding. Hence, he does not benefit from overbidding his cost.
\item Underbid of cost ($b_i < c_i$):\\
Here there could be two possibilities:
\begin{enumerate}[label=(\roman*),leftmargin=2mm,  labelwidth=-3mm]
\item $b_i \geq b_{\overline{k_t}}$: This scenario is identical to case 1(a) shown above and hence there is no incentive for the agent to bid in this manner.
\item $b_i < b_{\overline{k_t}}$: 
With such a bid, the agent $i$ is able to enter the active set of allocated workers. \\
Let the position of the agent $i$ in the new active set $A'$ be $j$, that is, 
$o_i(b_i, b_{-i}) = j$,  and the agent with the highest bid in $A'$ is $\overline{k_t}' \leq \overline{k_t}$. Therefore, by underbidding his cost, agent $i$ is able to move the workers $p \in \{\overline{k_t}'+1, \cdots,  \overline{k_t}\}$ out of the active set. We now show that such a bid does not fetch agent $i$ an  increased utility.
As per the payment structure, 
\begin{align}
\label{payment-outofA-overbid}
p_i&(b_i, b_{-i}) = \sum_{s=  \overline{k_t}'}^{\overline{k_t}} x_{i,s}^{\ext}(b_i, b_{-i}) b_s \nonumber\\
& +  \sum_{s = \overline{k_t}+1}^{N} x_{i,s}^{\ext}(b_i, b_{-i}) b_s \nonumber\\& + \left(x_i^{\pst}(b_i, b_{-i})  - \sum_{s=\overline{k_t}' , s \neq i}^N x_{i,s}^{\ext}(b_i, b_{-i})\right) \overline{c} 
\end{align}
The second term in \cref{payment-outofA-overbid} is zero, this is due to the fact that in absence of agent $i$ $\{1,2,\ldots, \overline{k_t}\}$ can complete the current job $t$. Therefore, with even an underbid $i$ has no externality on agents $\{\overline{k_t}+1,\ldots,n\}$.
The third term in \cref{payment-outofA-overbid} is also zero as the allocation with truthful bidding was enough to complete the job $t$ by agents $\{1,2,\ldots, \overline{k_t}\}$. Hence in absence of $i$ the allocation with underbid $x_i^{\pst}(b_i,b_{-i})$ is met by the externality sum. By underbidding, the agent $i$ is therefore able to obtain the portions of the job which would have been allocated to $s \in \{\overline{k_t}'+1 ,\cdots, \overline{k_t}\}$. For all such agents $s$, $x_{i,s}^{\ext}(b_i, b_{-i}) >0$, since in the absence of $i$, these agents would have received an allocation. But note that  $b_s < c_i$ and so these agents contribute towards a negative utility.
%
%As far as the agents $s \in \{\overline{k_t}+1, \cdots, n\}$ are concerned, a truthful bid from agent $i$ does not make a difference and
%$x_{i,s}^{\ext}(c_i, b_{-i}) = 0$.
%Also, \begin{align*}
%\sum_{s= \overline{k_t}'+1}^{\overline{k_t}} x_s^{\pst}(c_i, b_{-i}) &= \sum_{s= \overline{k_t}'+1}^{\overline{k_t}} x_{i,s}^{\ext}(b_i, b_{-i})\\
%&= x_i^{\pst} (b_i, b_{-i})
%\end{align*} and hence, 
%$x_{i,s}^{\ext}(b_i, b_{-i}) = 0$ for all  $s \in \{\overline{k_t}+1, \cdots, n\}$.
%Therefore, the third term in \cref{payment-outofA-overbid} vanishes.
Therefore the net utility $u_i(b_i, b_{-i}; c_i) < u_i(c_i, b_{-i}; c_i)$.
% \item Now we consider the case where $o_i(b_i, b_{-i}) < \overline{k_t}'$, where $\overline{k_t}' \leq \overline{k_t}$ is the last worker in the active set $A'$. This implies that the agent $i$ moves to a position well before the last worker by underbidding his cost. Without loss of generality let  $\overline{k_t}' $

\end{enumerate}
\end{enumerate}
\item Case 2: $i = \overline{k_t}$.\\
When agent $i$ bids truthfully, the active set is as follows:
\begin{align*}
\underbrace{b_1 \leq \cdots \leq b_{j} \leq \cdots \leq c_i }_{\text{Bids from} A} \leq b_{i+1} \cdots 
\end{align*}
and the payment to agent $i$
\begin{align}
\label{payment-last-agent-truthful}
p_i(c_i,&b_{-i}) = \sum_{s= i+1}^N x_{i,s}^{\ext} (c_i, b_{-i}) b_s\nonumber\\ &+ \left( x_i^{\pst}(c_i, b_{-i}) - \sum_{s=i+1}^n x_{i,s}^{\ext}(c_i, b_{-i})\right)\overline{c}
\end{align}
\begin{enumerate}[leftmargin=1mm, label=\alph*), labelwidth=-2mm]
\item Overbid of cost ($b_i > c_i$): \\
Here we look at two possible values of the range of the bids.
\begin{enumerate}[label=(\roman*), leftmargin=3mm, labelwidth=-4mm]
\item \textit{An overbid such that agent $i$ no longer belongs to the active set $A'$}:  At the position $o_i(b_i, b_{-i})$, the allocation to him is zero, that is,  $x_i^{\pst}( b_i, b_{-i}; c_i) = 0$ and thereby the utility from overbidding would be less than the utility from truthful bidding. Hence, he does not benefit from overbidding his cost in this manner.
\item \textit{An overbid such that agent $i$ remains in the active set but brings in other higher cost agents into the active set}:
Suppose the active set $A'$ contains the agents $\{i+1, \cdots, p \}$ in addition to the set $A$, such that, without loss of generality, 
\begin{align*}
\underbrace{b_1 < \cdots < b_{i-1} < b_{i+1} < \cdots < b_p < b_i}_{\text{Bids from} A'} < b_{p+1} < \cdots 
\end{align*}
The payment to agent $i$ with overbid is,
\begin{align}
\label{payment-overbid-newagents-lastagent}
p_i(b_i,&b_{-i}) = \sum_{s= p+1}^N x_{i,s}^{\ext} (b_i, b_{-i}) b_s\nonumber\\ &+ ( x_i^{\pst}(b_i, b_{-i}) - \sum_{s=p+1}^n x_{i,s}^{\ext}(b_i, b_{-i}))\overline{c}
\end{align}
Since the agents $\{i+1, \cdots, p \}$ have moved before $i$ in the ordering of the bids, those agents do not contribute to $p_i(b_i,b_{-i})$ further. However, for the agents $s \in \{p+1, \cdots, N\}$, $x_{i,s}^{\ext} (b_i, b_{-i}) =x_{i,s}^{\ext} (c_i, b_{-i})$ because the same proportion of job must be reassigned to the agent $s$ when $i$ bids $b_i$ as well as when $i$ is truthful. The first term in \cref{payment-last-agent-truthful} therefore strictly exceeds first term in \cref{payment-overbid-newagents-lastagent}.  
We now show that the second terms in \cref{payment-last-agent-truthful} and \cref{payment-overbid-newagents-lastagent} are equal.
Observe that,
\begin{align*}
x_i^{\pst}(c_i, b_{-i}) &= x_i^{\pst}(b_i, b_{-i}) + \sum_{s=i+1}^p x_{j}^{\pst}(b_i, b_{-i})\nonumber \\
 &= x_i^{\pst}(b_i, b_{-i}) + \sum_{s=i+1}^p x_{i,s}^{\ext}(b_i, b_{-i}) 
\end{align*}
A simple substitution for $x_i^{\pst}(b_i, b_{-i})$ in \cref{payment-overbid-newagents-lastagent} shows that the second terms in 
\cref{payment-last-agent-truthful} and \cref{payment-overbid-newagents-lastagent} are equal.
Therefore the overall payment $p_i(b_i, b_{-i}) < p_i(c_i, b_{-i})$ and further $u_i(b_i,b_{-i}; c_i) < u_i(c_i, b_{-i}; c_i)$.
\end{enumerate}
\item Underbid of cost ( $b_i < c_i$):
Note that in this scenario, there are the following two possibilities.
\begin{enumerate}[label=(\roman*),leftmargin=1mm, labelwidth=-3mm]
\item The active set $A'$ = $A$. The agent $i$ moves to a new position $j$, that is, $o_i(b_i, b_{-i}) = j$. Without loss of generality, we can consider that the agent with the highest bid in $A'$ is now agent $i - 1$.
The ordering of the agents is now,
\begin{align*}
\underbrace{b_1 \leq \cdots \leq b_{j-1} \leq \bold{b_i} \leq b_j \leq b_{i-1}}_{\text{Bids from} A'} \leq b_{i+1} \cdots .
\end{align*}

 By our payment structure,
\begin{align}
\label{payment-underbid-last-item}
p_i(b_i,&b_{-i}) = \sum_{\substack{ s= i-1 \\ s \neq i}}^N x_{i,s}^{\ext} (b_i, b_{-i}) b_s\nonumber\\ &+ ( x_i^{\pst}(b_i, b_{-i}) - \sum_{\substack{ s= i-1 \\ s \neq i}}^N x_{i,s}^{\ext}(b_i, b_{-i}))\overline{c}
\end{align}
Since the active set remains the same in spite of underbidding, $\forall s, \; i+1 \leq s \leq N$, $x_{i,s}^{\ext}(b_i, b_{-i}) = x_{i,s}^{\ext} (c_i, b_{-i})$ and in addition, $x_{i, i-1}^{\ext}(b_i, b_{-i}) > 0$, but, $b_{i-1} < c_i$.  Therefore, the first term in \cref{payment-last-agent-truthful} exceeds the first term in \cref{payment-underbid-last-item}. 
We also know that, 
\begin{equation}
x_i^{\pst}(b_i, b_{-i}) = x_i^{\pst}(c_i, b_{-i}) + x_{i, i-1}^{\ext}(b_i, b_{-i}) 
\end{equation}
 since the additional allocation that $i$ gets due to an overbid would be allocated to the last agent $i-1$ in $A'$, in the absence of $i$. A simple substitution in \cref{payment-underbid-last-item} shows that the second terms in \cref{payment-last-agent-truthful} and  \cref{payment-underbid-last-item} are equal. Therefore $u_i(b_i, b_{-i}; c_i) < u_i(c_i, c_{-i}; c_i)$.
\item The active set $A'$ due to underbidding by agent $i$ is smaller than the active set $A$ due to truthful bidding by agent $i$: This means that some agents get removed from $A$.
Suppose the agents  $ s \in \{j+1, \cdots, i-1\}$ get pushed out in the active set $A'$. Then by a similar argument as in the case 2 (b) (i) above, $x_{i,s}^ {\ext}(b_i, b_{-i}) > 0 $, but $b_s < c_i$. Therefore these agents contribute towards a negative utility and hence, $u_i(b_i, b_{-i}; c_i) < u_i(c_i, c_{-i}; c_i)$.
\end{enumerate}
\end{enumerate}
\item Case 3: $i < \overline{k_t}$
\begin{enumerate}[label=\alph*), leftmargin=1mm, labelwidth=-3mm]
\item Overbid of cost ($b_i > c_i$):\\
If the agent $i$ bids a higher cost, the position of $i$ in the ranking order changes to one of the following.
\begin{enumerate}[label=(\roman*), leftmargin=1mm, labelwidth=-3mm]
\item $i \leq o_i(b_i, b_{-i}) < \overline{k_t}$:
 The allocation to the worker remains the same as when he is truthful, that is, $x_i^{\pst} (b_i, b_{-i}) $ $= x_i^{\pst} (c_i, b_{-i})$. Our payment structure ensures that the payment $p_i(b_i, b_{-i
}) = p_i(c_i, b_{-i})$ and hence $u_i(b_i, b_{-i}; c_i) = u_i(c_i, b_{-i}; c_i)$.
\item $o_i(b_i, b_{-i}) = \overline{k_t}$:
In this case, agent $i$ ends up losing a part of $x_i^{\pst}(c_i, b_{-i})$ to the worker $k_t$. This scenario is analogous to Case 2 (a) (ii) where a worker who bids truthfully would have been at the last position $k_t$, but by overbidding ends up sharing his allocation with other agents. Therefore $u_i(b_i, b_{-i}; c_i) < u_i(c_i, b_{-i}; c_i)$.

\item $ o_i(b_i, b_{-i}) > \overline{k_t}$: Here, agent $i$ does not receive any allocation and thereby his payment as well as utility are both zero.
\end{enumerate}
\item Underbid of cost ($b_i < c_i$):\\
 Upon bidding a lower cost, the agent moves further up in the ranking order, that is $o_i(b_i, b_{-i}) \leq i$. The allocation also does not change, that is, $x_i^{\pst} (b_i, b_{-i}) = x_i^{\pst} (c_i, b_{-i})$. Our payment structure ensures that the payment $p_i(b_i, b_{-i
}) = p_i(c_i, b_{-i})$ and hence $u_i(b_i, b_{-i}; c_i) = u_i(c_i, b_{-i}; c_i)$.
\end{enumerate}
\end{itemize}
\emph{Future rounds}: If the agent $i$ ignores the loss incurred in the current job $t$ and chooses to manipulate the current bid for future utility, the resulting argument rolls back to one the above three cases.\hfill \qed
\end{proof}

\section{Regret Analysis}
\label{sec:regret-analysis}
\noindent In the strategic as well as the non-strategic settings, the underlying optimization problem involves parameters that are learnt in tandem. Hence regret is an important notion which we analyse in this section. Following are some relevant definitions. A problem instance in this space is characterized by a set of crowd agents $N$, the vector $c$ of their costs, the mean vectors ($\rho, \beta$), and the design parameters -- Deadline($D$), accuracy ($\varepsilon$). 

\begin{definition}{Optimal worker set:} For a problem instance with all the parameters known, in the solution to the optimization problem of \cref{opt_problem_known}, we refer to the set of agents allocated non-zero fraction of the job as the the optimal worker set.
\end{definition}

\begin{definition}{Optimal allocation:} We refer to the solution of \cref{opt_problem_known} as the optimal allocation.
\end{definition}

\begin{definition}{$\Delta$-Separation:} Let $k^*$ be the agent in the optimal worker set with the highest bid. In the optimal allocation (social welfare maximizing), all workers' allocation except $k^*$ would meet the constraints in \cref{opt_problem_known} with equality. We refer to the $\Delta$-separation as the additional fraction of the job which agent $k^*$ can take without violating any of the constraints. As all the stochastic parameters in this space are continuous, almost surely $\Delta >0$.
\end{definition}

\begin{definition}{Regret:} A learning mechanism in this space suffers a loss in social welfare due to either a) non-optimal set selection or b) due to suboptimal allocation within the optimal set. Formally, regret of a mechanism $\mathcal{A}$, is given by
\[R(\mathcal{A}) = T \sum_{i=1}^n c_ix_i^* - \sum_{t=1}^T \sum_{i=1}^n c_ix_i^{(\mathcal{A})}(t), \]
where $x_i^{(\mathcal{A})}(t)$ is the allocation to the agent $i$ for the job $t$ by the mechanism $\mathcal{A}$.
\end{definition}

\noindent We use the truncated empirical estimator within our Robust UCB scheme. Through an invocation of the Bernstein inequality, we have, with high confidence (probability $> 1-t^{-4}$ for the $t^{th}$ job), the true mean lies within the Robust UCB and LCB indices(see Lemma 1 in \cite{DBLP:BUBECK12HEAVYTAIL}). With enough samples, the symmetric indices of the Robust UCB scheme shrinks small enough so that no additional agents than the optimal set are required to meet the spill-over even due to the pessimistic strategy used.

\begin{theorem}
The \TDUCB\ mechanism selects an optimal set after the job $t' \in O(\log T)$.
\label{thm:optpull}
\end{theorem}

\begin{proof}
We denote $k^*$ as the costliest agent in the optimal set. Let $x^* = \{x_1^*,\ldots,x_{k^*}^*\}$ denote the allocations when the means are known.  Consider $\Delta$,
$$\Delta = \frac{\min(D,\beta_{k^*} \log(\frac{1}{1-\varepsilon}))}{\rho_{k^*}} - x_{k^*}(t,\beta)$$
$\Delta$ denotes the additional fraction of work the agent $k^*$ can take up without violating the constraints.  Following is a sufficient condition on $t$ when the set selected by the pessimistic estimate matches the optimal set. 
\begin{equation}
\begin{array}{|c|}
\hline\vspace{-1ex}\\
\vspace{-2ex}
\text{Need to get :}x^*_i(\rho,\beta) - x_i^{\pst}(t) \leq \frac{\Delta}{n},\; \forall i \leq k^*\\ \\
\hline
\end{array}
\label{eq:suffcond}
\end{equation}
We denote $\rho_{\ri}(t)$ and $\beta_{\ri}(t)$ as the Robust UCB indices of the JCT and the TTF of a worker in the active set. Recall, the active set for job $t$ denotes an agent allocated a non-zero fraction of the job. The expression for the pessimistic allocation $  \forall i \in \{1,\ldots,k^*\}$ at job $t$ is given by
\begin{align}
x_i^{\pst}(t) \times (\hat{\rho}_i + \rho_{\ri}) = \min\left(D , (\hat{\beta}_i- \beta_{\ri}) \log\left(\frac{1}{1-\varepsilon}\right)\right)
\label{eq:pess}
\end{align}
The allocation $x_i^{\pst}$ is determined by equality in \cref{eq:pess} whenever the set chosen is not optimal. We analyse this allocation via  two cases to determine the job $t$ when condition in \cref{eq:suffcond} is met.\\
\noindent Case (i): $x_i^{\pst}(t)\times(\hat{\rho}_i + \rho_{\ri}) = D$ or $ x_i^{\pst}(t)  = \frac{D}{\hat{\rho}_i + \rho_{\ri}} $ Consider,
\begin{align*}
&x^*_i(\rho,\beta) - x_i^{\pst}(t) \\
 &= \frac{\min\left(D, \beta_i \log \left(\frac{1}{1-\varepsilon}\right) \right)}{\rho_i} - \frac{D}{\hat{\rho}_i + \rho_{\ri}} \\
&\leq \frac{D}{\rho_i} - \frac{D}{\hat{\rho}_i + \rho_{\ri}}\\
&\leq \frac{D}{\hat{\rho}_i - \rho_{\ri}} - \frac{D}{\hat{\rho}_i + \rho_{\ri}}  [\because  \text{w.h.p. $\rho_i \geq \hat{\rho}_i - \rho_{\ri}$}] \\
& \leq \frac{2D \rho_{\ri}}{\barbelow{\rho}^2}
\end{align*}
In the current case, we have that the sufficiency condition is met whenever $\rho_{\ri} \leq \frac{\Delta \barbelow{\rho}^2}{2nD}$. In terms of the job $t$, through the expression of robust UCB index $\rho_{\ri}$, the case is met whenever,
\begin{equation}
t \geq \frac{336 u_\rho n^2 D^2 \log(T)}{\Delta^2 \barbelow{\rho}^4}
\end{equation}
where $u_\rho$ is an upper bound on the second moment of completion time (can be shown via easy computation).\\
\noindent Case (ii): $x_i^{\pst}(t)\times(\hat{\rho}_i + \rho_{\ri}) = (\hat{\beta}_i - \beta_{\ri}) \log \left(\frac{1}{1-\varepsilon}\right)$ or $ x_i^{\pst}(t)  = \frac{\hat{\beta}_i - \beta_{\ri}}{\hat{\rho}_i + \rho_{\ri}}\log \left(\frac{1}{1-\varepsilon}\right) $. Unlike $\rho_{i}$ where samples are obtained for every $t$, the samples from surrogate are obtained after multiple (yet finite due to bounded $\beta$) jobs.  To simplify our analysis, we consider as if a sample of $\beta_i$ is obtained for every job, the difference due to this simplification is only within constant factors. Consider,
\allowdisplaybreaks
\begin{align*}
& x^*_i(\rho,\beta) - x_i^{\pst}(t) \\
 &= \frac{\min\left(D, \beta_i \log \left(\frac{1}{1-\varepsilon}\right) \right)}{\rho_i} 
- \frac{\hat{\beta}_i - \beta_{\ri}}{\hat{\rho}_i + \rho_{\ri}}\log \left(\frac{1}{1-\varepsilon}\right) \\
&\leq \left(\frac{\beta_i}{\rho_i} - \frac{\hat{\beta}_i - \beta_{\ri}}{\hat{\rho}_i + \rho_{\ri}}\right)\log \left(\frac{1}{1-\varepsilon}\right)\\
&\leq \left(\frac{\hat{\beta}_i + \beta_{\ri}}{\hat{\rho}_i - \rho_{\ri}} - \frac{\hat{\beta}_i - \beta_{\ri}}{\hat{\rho}_i + \rho_{\ri}}\right)\log \left(\frac{1}{1-\varepsilon}\right) \\
&\intertext{W.h.p $\rho_i \geq \hat{\rho}_i - \rho_{\ri}$ and $\beta_i \leq \hat{\beta}_i + \beta_{\ri} $. As the surrogate observes at most one sample for every sample of $\rho_i$, we have $\beta_{\ri} \geq \rho_{\ri}$}\\
&\leq \beta_{\ri} \times \frac{2\bar{\rho}+2\bar{\beta}}{\barbelow{\rho}^2} \log \left( \frac{1}{1-\varepsilon} \right)
\end{align*}
This gives us that the sufficiency condition is met whenever $\beta_{\ri} \leq \frac{\Delta \barbelow{\rho}^2}{n(2\bar{\rho}+2\bar{\beta}) \log \left( \frac{1}{1-\varepsilon} \right) }$. In terms of the round $t$, 
the sufficiency condition is met whenever 
\begin{equation}
t \geq \frac{64 n^2 u_\beta \log(T) (2\bar{\rho}+2\bar{\beta})^2 \log^2\left( \frac{1}{1-\varepsilon}\right)}{\Delta^2 \barbelow{\rho}^4}
\label{eq:beta}
\end{equation}
where $u_\beta $ is an upper bound on the second moment on the TTF. From \cref{eq:pess} and \cref{eq:beta}, the optimal set is chosen after $O(\log(T))$ online jobs.
\end{proof}

\noindent As mentioned earlier, the regret in this setting arises first out of sub-optimal set selection and thereon out of sub-optimal allocation. Through \cref{thm:optpull}, we bound the number of jobs where sub-optimal set is chosen. The following theorem establishes the asymptotic efficiency of our learning scheme.

\begin{theorem}
Average regret of \TDUCB\ mechanism approaches zero asymptotically. 
\label{thm:asym}
\end{theorem}
\begin{prf}
WLOG due to \Cref{thm:optpull},  we will consider the case where the active set is the optimal set. Let $k^*$ be the last member in the active set. The average regret, for the job $t$, is then given by 
\[ R_{avg,t} = \sum_{i=1}^n \frac{c^{k^*}\left(x^*_i - x^{\pst}_i(t)\right)}{t}\]
Through steps similar to the proof of \cref{thm:optpull}, we have for a job $t$ for an agent $i$ either
\[ \frac{c^{k^*}\left(x^*_i - x^{\pst}_i(t)\right)}{t} \leq \frac{2D \rho_{\ri}}{t \times \barbelow{\rho}^2} \]
\[ \text{or } \frac{c^{k^*}\left(x^*_i - x^{\pst}_i(t)\right)}{t} \leq \frac{\beta_{\ri}}{t} \times \frac{2\bar{\rho}+2\bar{\beta}}{\barbelow{\rho}^2} \log \left( \frac{1}{1-\varepsilon} \right). \]
As both $\beta_{\ri} /t$ and $\rho_{\ri}/t$  approach zero, we have
\begin{align*}
\lim_{t\rightarrow \infty}R_{avg,t} = 0.  \qed
\end{align*}
\end{prf}

\section{Simulations}
\label{sec:simulations}
We have shown theoretical guarantees on the regret in the asymptotic sense. In practice, the constants pertaining to the same are unknown. Also, our algorithm focusses on the  optimization of social welfare. However, the payments to the workers involve externality to the workers which depend on the Robust UCB learning scheme. The learning scheme does not provide guarantees on the stochastic parameters for the sub-optimal workers. But the payments to the optimal worker set may also involve these parameters of the sub-optimal workers. Hence the simulations help us to form a fair idea about the same. 
Therefore, in order to investigate the efficacy of our algorithm in practical terms, we tested our method on  synthetically generated datasets. 

We simulated a set of 400 diverse workers with different costs, completion times and error rates.
We fixed the following values, $\underline{\rho} = 50$, $\overline{\rho} = 100$, $\underline{\beta} = 25$, $\overline{\beta} =  35$, $\underline{c} = 10$ and $\overline{c} = 100$. Out of the 400 workers, 250 high performing workers were simulated with $\rho$ values uniformly sampled from $[50, 75]$, $\beta$ values uniformly from $[30,35]$ and costs from $[10,50]$. A set of 150 mediocre workers were also simulated. These workers were simulated with $\rho$ values set to $100$, $\beta$ to  $25$, and costs to $100$. The deadline $D$ was fixed as $50$ and the error probability threshold ($\varepsilon$) for an agent allocated a task was set to $0.01$.

We checked the performance of our algorithm for a total of $10^6$ jobs. The baseline used for comparison is the omniscient greedy allocation scheme with the parameters $\beta$ and $\rho$ known. We refer to this baseline as `Optimal'. In the context of payments, the `Optimal' baseline refers to the minimum payment for inducing truthful reports for the greedy allocation aware of the means.

With increasing number of jobs, we observed that the total payment reduced and the negative social welfare decreased. The plots for the same are shown in \cref{fig1}(a) and \cref{fig1}(b).   The performance of \TDUCB\  improves towards `Optimal' when executed over more jobs. This shows that the learning improves over time and converges to the optimal set / optimal allocation.
\begin{figure}[!htb]
 \centering
    \begin{minipage}{.25\textwidth}
        \centering
       \includegraphics[scale=0.2]{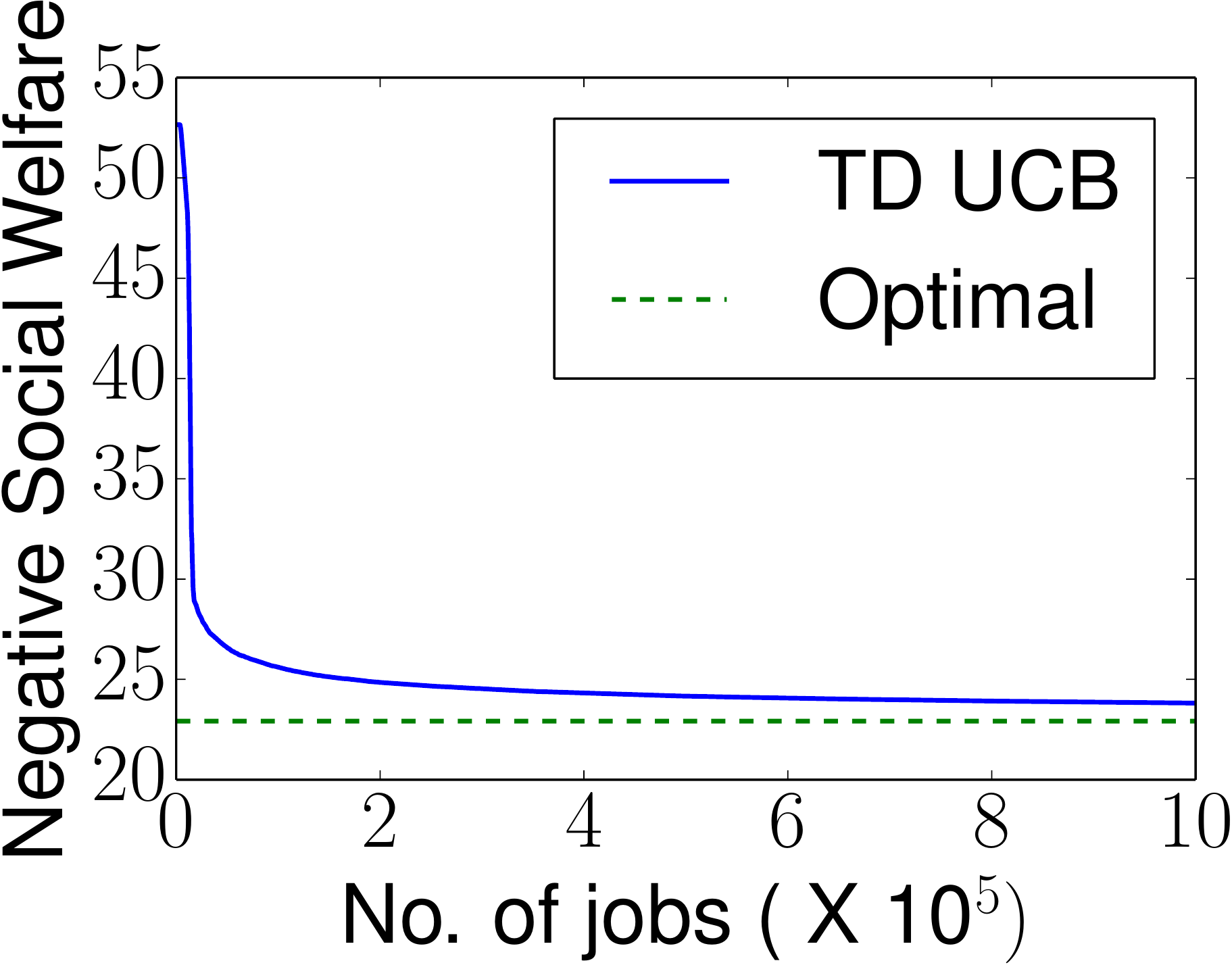}\\(a)
    \end{minipage}%
    \begin{minipage}{0.25\textwidth}
        \centering
        \includegraphics[scale=0.2]{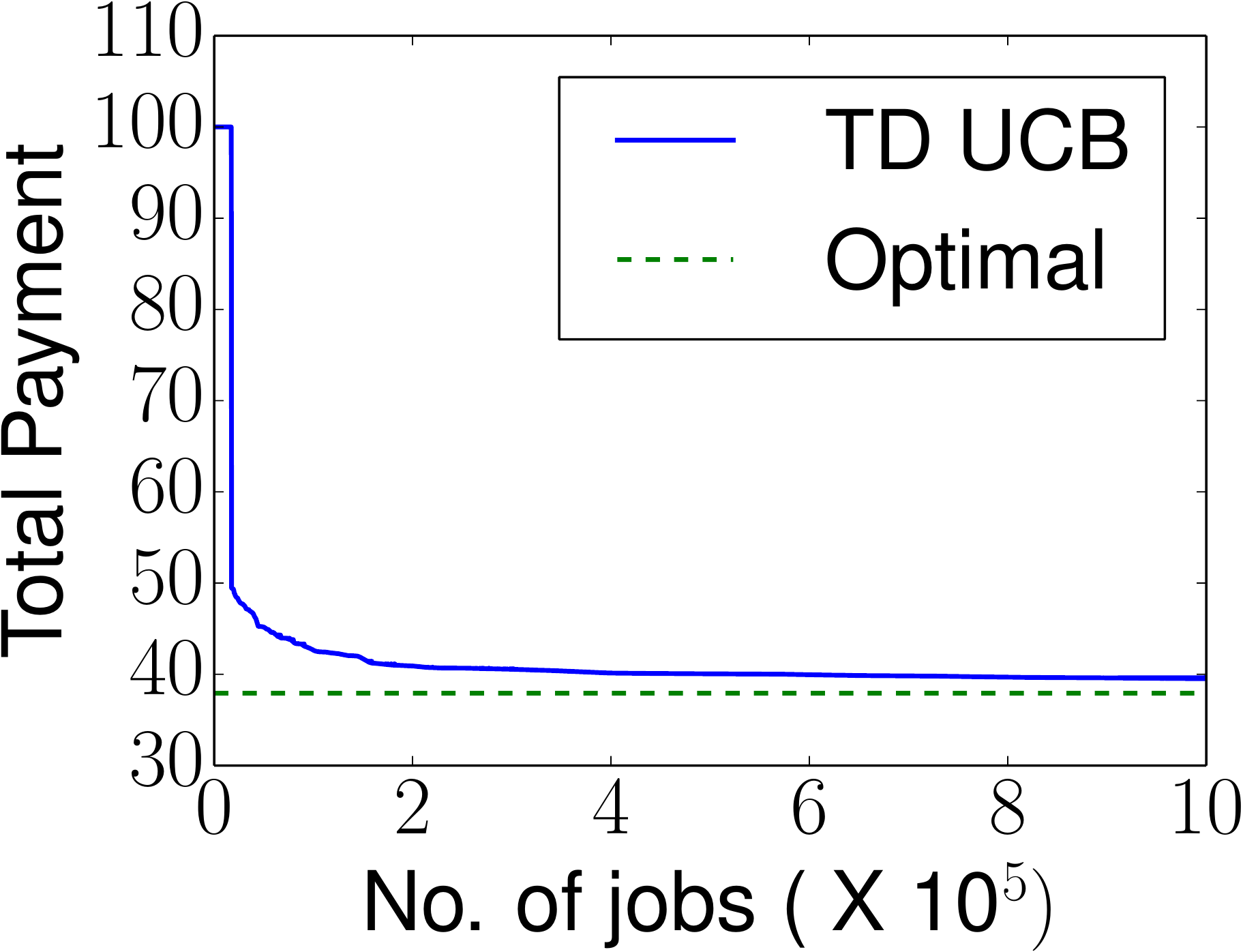}\\(b)
 \end{minipage}
     \caption{(a) Negative Social Welfare of the workers, (b) Total Payment to the workers. No. of workers = 400, $\varepsilon$ = 0.01}
     \label{fig1}
\end{figure}
\section{Future Work}
\noindent Following are some directions for future work. 
\begin{enumerate}[label=\alph*),leftmargin=1mm]
\item The derivation of theoretical bounds on the payments to the workers is still an open question. It would be interesting to see how these learned values affect the payment.
\item Investigate the use of other models for the JCT and TTF. For instance, the error committed by a worker could be modelled as a Bernoulli random variable the bias of which could depend on the time to completion. Such models would pose more challenges for the learning due to the interdependency between the stochastic parameters.
\item In our formulation of the problem in \cref{opt_problem_known}, we have posed the constraint on meeting the deadline as a requirement to be met in expectation. We would also be interested in satisfying the constraints in a probabilistic sense.
\end{enumerate}
\newpage
\bibliographystyle{abbrvnat}
\bibliography{crowd}
\end{document}